\documentclass[12pt]{article}
\usepackage[margin=1in]{geometry}
\usepackage{setspace}
\usepackage{microtype}
\usepackage{graphicx}

\usepackage{subcaption}
\usepackage{booktabs} 
\usepackage{makecell}
\usepackage{amsmath}
\usepackage{cprotect}
\usepackage{authblk}
\usepackage{fancyvrb}
\usepackage{indentfirst}
\usepackage{multirow}
\usepackage{enumitem}
\usepackage{xcolor}

\newcommand{\mc}[1]{\mathcal{#1}}

\usepackage{amsfonts}
\newcommand{\EnbPI}{\Verb|EnbPI|}
\newcommand{\ERAPS}{\Verb|ERAPS|}
\newcommand{\SRAPS}{\Verb|SRAPS|}
\newcommand{\SAPS}{\Verb|SAPS|}
\usepackage{array}
\newcolumntype{C}[1]{>{\centering\let\newline\\\arraybackslash\hspace{0pt}}m{#1}}
\newcommand{\kreg}{k_{reg}}
\usepackage{amsthm}

\newcommand{\R}{\mathbb{R}}
\newcommand{\PP}{\mathbb{P}}
\newcommand{\A}{\mathcal{A}}
\usepackage{stylesheet}
\usepackage{authblk}
\usepackage{hyperref}

\title{Conformal prediction set for time-series\thanks{Strongly accepted by the Workshop on Distribution-Free Uncertainty Quantification at International Conference on Machine Learning (ICML) 2022.} }

\author{Chen Xu, Yao Xie}

\begin{document}

\maketitle

\begin{abstract}
   When building either prediction intervals for regression (with real-valued response) or prediction sets for classification (with categorical responses), uncertainty quantification is essential to studying complex machine learning methods. In this paper, we develop Ensemble Regularized Adaptive Prediction Set (ERAPS) to construct prediction sets for time-series (with categorical responses), based on the prior work of \citep{EnbPI}. In particular, we allow unknown dependencies to exist within features and responses that arrive in sequence. Method-wise, ERAPS is a distribution-free and ensemble-based framework that is applicable for arbitrary classifiers. Theoretically, we bound the coverage gap without assuming data exchangeability and show asymptotic set convergence. Empirically, we demonstrate valid marginal and conditional coverage by ERAPS, which also tends to yield smaller prediction sets than competing methods.
\end{abstract}

\section{Introduction }
Despite the tremendous success of complex machine learning models such as neural networks for classification\citep{Krizhevsky2012ImageNetCW,Xie2017AggregatedRT,Qi2017PointNetDL}, how to construct prediction sets for uncertainty quantification is essential for the safe and confident deployment of these methods in practice. For example, in high-stake applications such as medical diagnosis \citep{Milletari2016VNetFC,Erickson2017MachineLF}, a single prediction of disease type can be insufficient; patients and doctors often want to be informed of several most likely possibilities for better treatment planning. This problem also arises in multiple hypothesis testing \citep{juditsky2020statistical} or the construction of least-ambiguous sets \citep{sadinle2019least}.
Nowadays, conformal prediction \citep{conformaltutorial, Zeni2020ConformalPA} has been one of the most popular uncertainty quantification frameworks for deep learning models. It is particularly appealing due to its distribution-free assumption of the data-generating processes and the applicability to arbitrary prediction models. In general, conformal prediction methods assign non-conformity scores to possible outcomes, some of which are included in the prediction set/interval depending on the value of the scores. Many methods follow this logic and have yielded promising results in fields of computer vision \citep{Candes_classification,MJ_classification}, drug discovery \citep{Eklund2013TheAO, Bosc2019LargeSC}, anomaly detection \citep{Smith2014AnomalyDO,Xu2021ConformalAD} etc.

However, the conformal prediction framework crucially depends on the assumption of data exchangeability, which can be a restrictive assumption in many settings. For our purpose of time-series classification, data exchangeability rarely holds. Therefore, how to extend conformal prediction beyond purely exchangeable data has sparked many research interests. For instance, \citep{CPcovshift} extends CP methods under distribution shift with fixed weights and suggests estimation techniques for the weights. \citep{Park2021PACPS} addresses covariate shifts by building probably approximately correct prediction sets, assuming the shifts in distributions are known. \citep{Stankeviit2021ConformalTF} explores exchangeability in time-series to provide distribution-free and model-free guarantees. More recently, \citep{Barber2022ConformalPB} proposes general results bounding the coverage gap of CP methods using total variation distances and suggests a simple weighting scheme for practical implementations.

In this paper, along these active lines of research, we  develop an algorithm \ERAPS \ for time-series prediction with categorical response, called the {\it time-series classification} problem,
provide theoretical guarantees, and demonstrate good empirical performances using numerical experiments. The current algorithm is built upon our prior work  \citep{EnbPI}, which develops a computationally efficient distribution-free ensemble-based method for time-series prediction with prediction interval. The contributions here include:
\begin{itemize}[itemsep=0em]
    \item Construct \ERAPS, an efficient ensemble-based conformal prediction framework for time-series classification. In particular, \ERAPS \ allows arbitrary dependency among the features and response and can outperform non-ensemble-based ones based on numerical experiments.
    \item Theoretically bound coverage gaps and verify convergence of the estimated prediction set without assuming data exchangeability. Theoretical results hold for arbitrary definitions of non-conformity scores.
\end{itemize}

\section{Problem Setup}
Let $(X_t,Y_t), t\geq 1$ be a collection of random variables with unknown distribution, where $X_t \in \R^d$ is the feature vector and $Y_t \in [K]:=\{1,\ldots,K\}$ is a discrete random variable. Denote $\pi:=P_{Y|X}$ as the true conditional distribution of $Y|X$, whose properties are unknown. In a typical classification setting, we assume the first $T$ data are known to us as training data and the goal is to construct an estimator $\hat{\pi}:=\A(\{(X_t,Y_t)\}_{t=1}^T)$, which satisfies $\sum_{c=1}^K \hat{\pi}_{X_t}(c)=1, \hat{\pi}_{X_t}(c)\geq 0$ for any $t\geq 1$. Here, $\A$ is any classification algorithm, from the simplest multinomial logistic regression to a complex deep neural network. Then, the point prediction $\hat{Y}_{t}:=\arg \max_{c\in [K]} \hat{\pi}_{X_t}(c)$ is obtained for any test index $t>T$.

However, there is no guarantee that the point estimate $\hat{Y}$ is close to the actual $Y$ when $\pi$ and the distribution of $(X,Y)$ are both unknown. As a result, there is inherent uncertainty in the estimate $\hat{Y}$, which is important to be quantified. In this work, we want to construct uncertainty sets $C(X_t,\alpha)$ at level $\alpha$ such that conditioning on $X_t=x_t$,
\begin{equation}\label{def:cond_cov}
    \PP(Y_t \in C(X_t,\alpha)|X_t=x_t)\geq 1-\alpha,
\end{equation}
where (\ref{def:cond_cov}) ensures valid \textit{conditional} coverage. In comparison, valid \textit{marginal} coverage requires the sets to satisfy 
\begin{equation}\label{def:marg_cov}
    \PP(Y_t \in C(X_t,\alpha))\geq 1-\alpha.
\end{equation}
A common example for (\ref{def:cond_cov}) and (\ref{def:marg_cov}) lies in image classification, where one builds prediction sets for a point estimate to reach classification with confidence. The former conditional guarantee is a stronger requirement and trivially implies the latter. Meanwhile, various works have proposed methods that reach (\ref{def:marg_cov}), requiring nothing but that $\{(X_i,Y_i)\}_{i=1}^T$ are exchangeable observations. Some works also leverage the power of pre-trained deep neural networks to build small and adaptive prediction sets \citep{MJ_classification}. However, reaching (\ref{def:cond_cov}) in a distribution-free setting is impossible \citep{DistFreeCond}, unless asymptotically \citep{CPHist}. Moreover, neither do most of the earlier theoretical guarantees for reaching (\ref{def:marg_cov}) apply when the observations are dependent (e.g., images by satellites that record the temporal evolution of earth conditions or recording of pedestrian walking patterns in different city locations at different times.) To fill in the gap, we build on \citep{EnbPI_old,EnbPI} to first present a general theoretical framework that bounds non-coverage of (\ref{def:cond_cov}) in Section \ref{sec:theory} for \textit{any} non-conformity score. Then, we develop the computationally efficient algorithm \ERAPS \ in Section \ref{sec:algo}, which builds ensemble classifiers and utilizes novel non-conformity scores \citep{MJ_classification} that work well in practice. Section \ref{sec:exper} demonstrates the performance of \ERAPS \ against competing methods on real data. Section \ref{sec:conclude} concludes the work and discusses future directions.

\section{Theoretical Guarantee}\label{sec:theory}

Instead of assuming that observations $(X_t,Y_t)$ are exchangeable, we impose assumptions on the quality of estimating the non-conformity scores and on the dependency of non-conformity scores in order to bound coverage gap of (\ref{def:cond_cov}). Note that most of the assumptions and proof techniques are similar to those in \citep{EnbPI}, but we extend it to the classification setting under arbitrary definitions of non-conformity scores. In particular, we allow arbitrary dependency to exist within features $X_t$ or responses $Y_t$. Proofs are contained in Section \ref{sec:proof}.

Given any feature $X$, a possible label $c$, and a probability mapping $p$ such that $\sum_{c=1}^K p_X(c)=1, p_X(c)\geq 0$, we denote $\N:(X,c,p)\rightarrow \mathbb{R}$ as an arbitrary non-conformity mapping and $\tau^p_X(c):=\N(X,c,p)$ as the non-conformity score at label $c$. For instance, we may consider 
\begin{equation}\label{ex:NCM}
    \N(X,c,p)=\sum_{c'=1}^K p_{X}(c')\cdot \textbf{1}(p_{X}(c')>p_{X_t}(c)),
\end{equation}
which computes the total probability mass of labels that are deemed more likely than $c$ by $p$. The less likely $c$ is, the greater $\tau^p_i(c)$ is, indicating the non-conformity of label $c$. For notation simplicity, the oracle (resp. estimated) non-conformity score of each training datum $(X_i,Y_i), i=1,\ldots,T$ under the true conditional distribution $\pi:=P_{Y|X}$ (resp. any estimator $\hat{\pi}$) is abbreviated as $\tau_i=\tau^{\pi}_{X_i}(Y_i)$ (resp. $\hat \tau_i$). 

Then, the prediction set $C(X_t,\alpha)$ for $t>T$ is defined as
\begin{equation}\label{eq:set}
    C(X_t,\alpha):=\{c\in [K]:  \sum_{j={t-T}}^{t-1} \textbf{1}(\hat \tau_j\leq \hat{\tau}_t(c))/T < 1-\alpha\},
\end{equation}
which includes all the labels whose non-conformity scores are no greater than $(1-\alpha)$ fraction of previous $T$ non-conformity scores. We now impose these two assumptions that are sufficient for bounding coverage gap of (\ref{def:cond_cov}):
\begin{assumption}[Error bound on estimation]\label{assumption:estimation}
Assume there is a real sequence $\{\gamma_T\}$ whereby
\[
 \sum_{j={t-T}}^{t-1} (\hat \tau_j-\tau_j)^2/T\leq \gamma_T^2.
\]
\end{assumption}
\begin{assumption}[Regularity of non-conformity scores]\label{assumption:regularity}
Assume $\{\tau_j\}_{j={t-T}}^t$ are independent and identically distributed (i.i.d.) according to a common cumulative density function (CDF) $F$ with Lipschitz continuity constant $L>0$.
\end{assumption}
We brief remark on implications of the Assumptions above. Note that Assumption \ref{assumption:estimation} essentially reduces to the point-wise estimation quality of $\pi$ by $\hat{\pi}$, which may fail under data overfitting---all $T$ training data are used to train the estimator. In this case, $\hat{\pi}$ tends to over-concentrate on the empirical conditional distribution under $(X_i,Y_i), i=1,\ldots,T$, which may not be representative of the true conditional distribution $P_{Y|X}$. A common way to avoid this in the CP literature is through data-splitting---train the estimator on a subset of training data and compute the estimated non-conformity scores $\hat{\tau}$ only on the rest training data (i.e., calibration data). However, doing so likely results in a poor estimate of $\pi$ and as we will see, the theoretical guarantee heavily depends on the size of estimated non-conformity scores. On the other hand, Assumption \ref{assumption:regularity} can be relaxed as stated in \citep{EnbPI}. For instance, the oracle non-conformity scores can either follow linear processes with additional regularity conditions \citep[Corollary 1]{EnbPI} or be strongly mixing with bounded sum of mixing coefficients \citep[Corollary 2]{EnbPI}. The proof techniques directly carry over, except for slower convergence rates. 

Lastly, define the following empirical distributions using oracle and estimated non-conformity scores:
\begin{align*}
    & \tildeF{x}:=\sum_{j={t-T}}^{t-1} \textbf{1}(\tau_j\leq x)/T && \text{[Oracle]}\\
    & \hatF{x}:=\sum_{j={t-T}}^{t-1} \textbf{1}(\hat \tau_j\leq x)/T && \text{[Estimated]}
\end{align*}

We then have the following conditional coverage results at the prediction index $t>T$.
\begin{lemma}\label{lem:tildeFandhatF}
Suppose Assumptions \ref{assumption:estimation} and \ref{assumption:regularity} hold. Then,
\[
\sup_{x}|\tildeF{x}-\hatF{x}|\leq (L+1)\gamma_T^{2/3}+2 \sup_{x}|\tildeF{x}-F(x)|.
\]
\end{lemma}

\begin{lemma}\label{lem:tildeFandF}
Suppose Assumption \ref{assumption:regularity} holds. Then, for any training size $T$, there is an event $A$ within the probability space of non-conformity scores $\{\tau_j\}_{j=1}^T$, such that when $A$ occurs,
\[
    \sup_{x}|\tildeF{x}-F(x)|\leq \sqrt{\log(16T)/T}.
\]
In addition, the complement of event $A$ occurs with probability $P(A^C)\leq \sqrt{\log(16T)/T}.$
\end{lemma}

As a consequence of Lemmas \ref{lem:tildeFandhatF} and \ref{lem:tildeFandF}, the following bound of coverage gap of (\ref{def:cond_cov}) holds:
\begin{theorem}[Conditional coverage gap bound, i.i.d. non-conformity scores]\label{thm:asym_cond_cov}
Suppose Assumptions \ref{assumption:estimation} and \ref{assumption:regularity} hold. For any training size $T$ and significance level $\alpha\in (0,1)$, we have 
\begin{equation}\label{eq:guarantee}
    |\PP(Y_t \notin C(X_t,\alpha)|X_t=x_t)-\alpha|\leq 24\sqrt{\log(16T)/T}+4(L+1)\gamma_T^{2/3}.
\end{equation}
\end{theorem}
Note that Theorem \ref{thm:asym_cond_cov} holds uniformly over all $\alpha \in [0,1]$ because Lemmas \ref{lem:tildeFandhatF} and \ref{lem:tildeFandF} bound the sup-norm of differences of distributions. Hence, users in practice can select desired parameters $\alpha$ \textit{after} constructing the non-conformity scores. Such a bound is also useful when building multiple prediction intervals simultaneously, under which $\alpha$ is corrected and lowered to reach nearly valid coverage \citep{Farcomeni2008ARO}.

The following bound on marginal coverage gap is a simple corollary, which holds by the law of total expectation (proof omitted).
\begin{corollary}[Marginal coverage gap bound, i.i.d. non-conformity scores]\label{thm:asym_marg_cov}
Suppose Assumptions \ref{assumption:estimation} and \ref{assumption:regularity} hold. For any training size $T$ and significance level $\alpha\in (0,1)$, we have 
\begin{equation}\label{eq:guarantee}
    |\PP(Y_t \notin C(X_t,\alpha))-\alpha|\leq 24\sqrt{\log(16T)/T}+4(L+1)\gamma_T^{2/3}.
\end{equation}
\end{corollary}

In addition to coverage guarantee, we can analyze the convergence of $C(X_t,\alpha)$ to the oracle prediction set $C^*(X_t,\alpha)$ under further assumptions. Given the true conditional distribution function $\pi:=P_{Y|X}$, we first order the labels so that $\pi_{X_t}(i)\geq \pi_{X_t}(j)$ if $i\leq j$. Then, we have
\[
C^*(X_t,\alpha)=\{1,\ldots,c^*\},
\]
where $c^*:=\min_{c\in [K]} \sum_{k=1}^c \pi_{X_t}(k)\geq 1-\alpha.$
\begin{theorem}[Asymptotic convergence to the true prediction set, i.i.d. non-conformity scores]\label{thm:asy_set}
Suppose Lemmas \ref{lem:tildeFandhatF} and \ref{lem:tildeFandF} hold and denote $F^{-1}$ as the inverse CDF of $\{\tau_j\}_{j=t-T}^t$. Further assume that 
\begin{itemize}
\item[(1)] $c^*_1=c^*_2$ where
\begin{align*}
    c^*_1&:= \arg\min_{c} \left\{\sum_{k=1}^c \pi_{X_t}(k) \geq 1-\alpha\right\}, \\
    c^*_2&:=\arg\max_{c} \left\{\tau_t(c) < F^{-1}(1-\alpha)\right\}.
\end{align*}

\item[(2)] There exists a sequence $\gamma'_T$ converging to zero with respect to $T$ such that $\|\tau_t-\hat{\tau}_t\|_{\infty} \leq \gamma'_T$, where the $\infty$-norm is taken over class labels.

\item[(3)] The sequence $\gamma_T$ converges to zero as in Assumption \ref{assumption:estimation}.
\end{itemize}

Then, there exists $T$ large enough such that for all $t\geq T$, 
\begin{equation}\label{set_diff}
    C(X_t,\alpha) \Delta C^*(X_t,\alpha) \leq 1,
\end{equation}
where $\Delta$ in \eqref{set_diff} denotes set difference.
\end{theorem}
Note that if the non-conformity score at any label $c$ is defined in \eqref{ex:NCM}, which is the total probability mass of labels $c'\neq c$ that are more likely than $c$ based on a conditional probability mapping $p$, then the first additional assumption (i,e., $c^*_1=c^*_2$) in Theorem \ref{thm:asy_set} can be verified to hold. In general, whether this assumption is satisfied depends on the particular form of the non-conformity score.

\section{Algorithm}\label{sec:algo}

Upon following earlier theoretical guarantees, we first specify a particular form of non-conformity score recently developed in \citep{MJ_classification} using any estimator $\hat{\pi}$. The notations are very similar and we include the descriptions for a self-contained exposition. We then describe our efficient algorithm \ERAPS \ in Algorithm \ref{algp:eraps}, which is ensemble-based and avoids data-splitting. Given the estimator $\hat{\pi}$, for each possible label $c$ at test feature $X_t, t>T$, we make two other definitions:
\begin{align}
 & m_{X_t}(c):=\sum_{c'=1}^K \hat{\pi}_{X_t}(c')\cdot \textbf{1}(\hat{\pi}_{X_t}(c')>\hat{\pi}_{X_t}(c)).\label{def:tot_mass}\\
   & r_{X_t}(c):=\left \vert\sum_{c'=1}^K \textbf{1}(\hat{\pi}_{X_t}(c')>\hat{\pi}_{X_t}(c))\right\vert+1. \label{def:rank}
\end{align}
In words, (\ref{def:tot_mass}) calculates the total probability mass of labels deemed more likely than $c$ by $\hat{\pi}$. It strictly increases as $c$ becomes less probable. Meanwhile, (\ref{def:rank}) calculates the rank of $c$ within the order statistics. It is also larger for less probable $c$. Given a random variable $U_t\sim \text{Unif}[0,1]$ and pre-specified regularization parameters $\{\lambda, \kreg\}$, we define the non-conformity score as
\begin{equation}\label{eq:tau_MJ}
    \tau^{\hat \pi}_{X_t}(c):=m_{X_t}(c)+\underbrace{\hat{\pi}_{X_t}(c)\cdot U_t}_{(i)} + \underbrace{\lambda(r_{X_t}(c)-\kreg)^+}_{(ii)}.
\end{equation}
We interpret terms (i) and (ii) in (\ref{eq:tau_MJ}) as follows. Term (i) randomizes the uncertainty set, accounts for discrete probability jumps when new labels are considered. A similar randomization factor is used in \cite[Eq. (5)]{Candes_classification}. In term (ii), $(z)^+:=\max(z,0)$. Meanwhile, the regularization parameters $\{\lambda, \kreg\}$ force the non-conformity score to increase when $\lambda$ increases and/or $\kreg$ decreases. In words, $\lambda$ denotes the additional penalty when the label is less probable by one rank and $\kreg$ denotes when this penalty takes place. This term ensures that the sets are \textit{adaptive}, by returning smaller sets for easier cases and larger ones for harder cases. The empirical performance of these regularization parameters is examined in the experiments (see Figure \ref{fig:marg_cov_regularizers}). The prediction set $C(X_t,\alpha)$ is then constructed based on (\ref{eq:set}) using definition \eqref{eq:tau_MJ}.

To train the estimator $\hat{\pi}$, \cite{MJ_classification} adopts a split-conformal framework: split the $T$ training data into two disjoint sets, where the first set is used to train the estimator and the second set is for computing the scores. Algorithm \ref{algo:sraps} includes the formal details, idential to ones in \citep[RAPS]{MJ_classification}. Notation-wise, for any set $S$ of scalars, $q_{S,1-\alpha}:=(1-\alpha) \text{ quantile of } S$.

\begin{algorithm}[h!]
\cprotect \caption{Split Regularized Adaptive Prediction Set (\SRAPS) \citep{MJ_classification}}
\label{algo:sraps}
\begin{algorithmic}[1]
\Require{Training data$\{(X_t, Y_t)\}_{t=1}^T$, classification algorithm $\mathcal{A}$, $\alpha$, regularization parameters $\{\lambda,\kreg\}$, and test features $\{X_t\}_{t=T+1}^{T+T_1}$.}
\Ensure{Uncertainty sets $\{C(X_t,\alpha)\}_{t=T+1}^{T+T_1}$}
\State Randomly split the $T$ training indices to $\mathcal{I}_1,\mathcal{I}_2$ such that data indexed by $\mathcal{I}_1$ are for training and data in $\mathcal{I}_2$ are for calibration.
\State Sample $\{U_t\}_{t=1}^{T+T_1}\overset{i.i.d.}{\sim}\text{Unif}[0,1]$.
\State Compute $\hat{\pi}:=\A(\{(X_t,Y_t)\}_{t\in \mathcal{I}_1})$.
\For {$t \in \mathcal{I}_2$}
\State Compute $\hat \tau_t:=\hat \tau_{X_t}^{\hat \pi}(Y_t)$ using (\ref{eq:tau_MJ}). 
\EndFor
\State Compute $\hat \tau_{cal}:=q_{\mathcal{I}_2,1-\alpha}(\{\hat \tau_t\}_{t \in \mathcal{I}_2})$.
\For {$t=T+1,\ldots,T+T_1$}
\State Compute $C(X_t,\alpha)$ in (\ref{eq:set}) using $\hat{\pi}$ and $\hat \tau_{cal}$.
\EndFor
\end{algorithmic}
\end{algorithm}

\begin{remark}[Limitations of Algorithm \ref{algo:sraps}]\label{remark:downside} There are two limitations with the split conformal formulation:
\begin{itemize}
\item[(1)] One has to partition the data into two subsets: one for training $\hat{\pi}$ and the other for calibration of $\hat{\tau}_{cal}$. Doing so likely results in larger uncertainty sets and poor estimation of the true conditional distribution, especially when data have dependency. More precisely, the term $T$ on the RHS of (\ref{eq:guarantee}) in Theorem \ref{thm:asym_cond_cov} becomes the size of the calibration data $|\mc I_2|$.

\item[(2)] Nowadays, ensemble predictors are widely used to increase estimation quality (thus reduce the size of uncertainty sets). However, using ensemble predictor within the split conformal framework may even increase the size of uncertainty sets, since each bootstrap predictor is trained on too few data.
\end{itemize}
\end{remark}

To resolve limitations in Remark \ref{remark:downside}, we resort to recent CP methods under ensemble learning \cite{J+,J+aB,EnbPI}, which have been successful in building prediction intervals for regression. In particular, we adapt the LOO ensemble predictor idea in \EnbPI \ to train $\hat{\pi}$ above \cite{EnbPI}. Algorithm \ref{algp:eraps} formally describes the details. Notation-wise, variables with superscript $\phi$ are results of aggregation via $\phi$.

\begin{algorithm}[h!]
\cprotect \caption{Ensemble Regularized Adaptive Prediction Set (\ERAPS)}
\label{algp:eraps}
\begin{algorithmic}[1]
\Require{Training data$\{(X_t, Y_t)\}_{t=1}^T$, classification algorithm $\mathcal{A}$, $\alpha$, regularization parameters $\{\lambda,\kreg\}$, aggregation function $\phi$, number of bootstrap models $B$, the batch size $s$, and test data $\{(X_t,Y_t)\}_{t=T+1}^{T+T_1}$, with $Y_t$ revealed only after the batch of $s$ prediction intervals with $t$ in the batch are constructed.}
\Ensure{Ensemble uncertainty sets $\{C(X_t,\alpha)\}_{t=T+1}^{T+T_1}$}
\For {$b = 1, \dots, B$}  \Comment{Train Bootstrap Estimators}
\State Sample with replacement an index set $S_b=(t_1,\ldots,t_T)$ from indices $(1,\ldots,T)$.
\State Compute $\hat{\pi}^b=\mathcal{A}(\{(X_t,Y_t) \mid t \in S_b \})$.
\EndFor
\State Initialize $\boldsymbol \tau=\{\}$ and sample $\{U_t\}_{t=1}^{T+T_1}\overset{i.i.d.}{\sim}\text{Unif}[0,1]$.
\For {$t=1,\dots,T$} \Comment{LOO Ensemble Estimators and Scores}
\State Compute $\hat{\pi}_{-t}^{\phi}:=\phi(\{\hat{\pi}^b: t \notin S_b \})$ such that for each $c\in \{1,\ldots,K\}$
\Statex \hskip \algorithmicindent$\hat{\pi}_{-t,X_t}^{\phi}(c)=\phi(\{\hat{\pi}^b_{X_t}(c): t \notin S_b \}).$
\State Compute $\hat \tau_t^{\phi}:=\hat \tau_{X_t}^{\hat{\pi}_{-t},\phi}(Y_t)$ using (\ref{eq:tau_MJ}). 
\State $\boldsymbol \tau=\boldsymbol \tau \cup \{\hat{\tau}_t^{\phi}\}$
\EndFor
\For {$t=T+1,\dots,T+T_1$} \Comment{Build Uncertainty Sets}
\State Compute $\hat \tau_{t,cal}^{\phi}:=q_{\boldsymbol \tau,1-\alpha}(\boldsymbol \tau)$.
\State Compute $\hat \pi_{-t}^{\phi}:=\phi(\{\hat \pi_{-i}^{\phi}\}_{i=1}^T)$ so that for each $c\in \{1,\ldots,K\}$
\Statex \hskip \algorithmicindent $\hat \pi_{-t,X_t}^{\phi}(c):=\phi(\{\hat \pi_{-i,X_i}^{\phi}(c)\}_{i=1}^T).$
\State Compute $C(X_t,\alpha)$ in (\ref{eq:set}) using $\hat \pi_{-t}^{\phi}$ and $\hat \tau_{t,cal}^{\phi}$.
\If {$t-T$ = 0 mod $s$} \Comment{Slide Scores Forward}
\For {$j=t-s,\ldots,t-1$}
\State Compute $\hat \tau_j^{\phi}:=\hat \tau_{X_j}^{\hat{\pi}_{-j},\phi}(Y_j)$ using (\ref{eq:tau_MJ}).
\State $\boldsymbol\tau=(\boldsymbol\tau- \{\hat{\tau}_1^{\phi}\}) \cup \{\hat{\tau}_j^{\phi}\}$ and reset index of $\boldsymbol\tau$.
\EndFor
\EndIf
\EndFor
\end{algorithmic}
\end{algorithm}

\begin{remark}[Class-conditional conditional coverage]\label{remark:class_cond}
In addition to the conditional coverage in (\ref{def:cond_cov}), one is often also interested in the following class-conditional coverage guarantee: 
\begin{equation}\label{def:class_cond_cov}
    \PP(Y_t \in C(X_t,\alpha)|Y_t = c)\geq 1-\alpha,
\end{equation}
where $c$ is the actual label of $Y_t$.
In many high-stake applications (e.g., criminal justice, medical diagnosis, financial loan), achieving (\ref{def:class_cond_cov}) can be important. For instance, given that the true label is a type of cancer, we want to make sure the prediction set contains that type so that the patient is well-informed.

We can ensure (\ref{def:class_cond_cov}) simply through applying \ERAPS \ separately on subsets of data categorized by classes, so that the marginal non-coverage gap bound in Theorem \ref{thm:asym_marg_cov} is equivalent to the conditional one. However, doing so can be computationally expensive when the number of classes is high. We hereby also present a computationally efficient strategy, when the fitted classifier $\hat \tau$ well approximates $\pi:=P_{Y|X}$ as in Assumption \ref{assumption:estimation}. For each class $c$, first compute $\boldsymbol \tau^c:=\{\tau_t: Y_t=c, t\in [T]\}$. Then, $\hat \tau^c_{t, cal}:=q_{\boldsymbol \tau^c,1-\alpha}(\boldsymbol \tau^c)$. Hence, rather than $\hat \tau_{t, cal}$, the set $\{\hat \tau^c_{t, cal}\}$ will instead be used to identify class labels that are conforming to earlier ones. During sliding, one would also only update $\boldsymbol \tau^c$ if the actual label $Y_t=c$. 

Theoretically, because the empirical distribution using $\boldsymbol \tau^c$ also well approximates the actual distribution of the non-conformity score \textit{when the true label of $Y_t$ is $c$}, it approximately reaches (\ref{def:class_cond_cov}) as well. In addition, the class-conditional thresholds $\hat \tau^c_{t,cal}$ may also help identify hard-to-predict classes because it is large only when the true label is $c$, but the predictors assign $c$ with low probability relative to other classes.

Lastly, note that one can also use $\hat{\tau}^{\max}_{cal}:=\max_{c\in[K]} \hat{\tau}^{c}_{t, cal}$ if the original marginal threshold fails to cover certain classes and the class-conditional scores at some classes are too small to ensure coverage. However, the prediction sets will likely be too conservative.
\end{remark}

\section{Experiments}\label{sec:exper}

We start by describing the datasets, competing methods, and classifier/hyperparameter settings. Regarding datasets, we examine \ERAPS \ on three time-series classification problems. The datasets are all publicly available: the first Melbourne pedestrian data (Pedestrian) predicts city location based on pedestrian flow activities and comes from the city of Melbourne\footnote{http://www.pedestrian.melbourne.vic.gov.au/\#date=11-06-2018\&time=4}. The second Crop classification dataset (Crop) predicts geographic regions based on satellite images \citep{Tan2017IndexingAC}. The last pen digit classification dataset predicts the hand-written digits by 44 writers \citep{Alimoglu1997CombiningMR}. Regarding competing methods, we compare with three conformal-prediction methods: the first is the \SRAPS \ in Algorithm \ref{algo:sraps}, the second is the \textit{split adaptive prediction set} \SAPS \ by \citep{Candes_classification}, and the last is the \textit{naive predictor}. The naive predictor considers the top$-k$ labels based on the estimator, where the cumulative likelihood for these $k$ labels exceeds $1-\alpha$. Regarding classifier/hyperparameter settings, we construct neural network classifiers with fully-connected layers (NN); other classifiers such as random forests or logistic regressions were also considered, but the overall patterns remain similar so that details are omitted. The mean aggregation function is used in \ERAPS, which builds 30 bootstrap estimators for each classifier. \SRAPS \ uses the last 50\% of the training data for calibration. These choices are guided by \citep{EnbPI}; different hyperparameter combinations/model choices do not affect the performances much.

In terms of performance assessment, we first examine the marginal coverage and set sizes by different methods in Section \ref{exp:marginal}. We then analyze the conditional coverage and set sizes under two cases in Section \ref{exp:conditional}. 
Code is available at \url{https://github.com/hamrel-cxu/Ensemble-Regularized-Adaptive-Prediction-Set-ERAPS}.

\subsection{Marginal results}\label{exp:marginal}

\begin{table}[t]
\centering
\cprotect\caption{Marginal coverage and set size on different datasets, where we choose $\alpha \in \{0.05,0.075,0.1,0.15,0.2\}$. We see that \ERAPS \ almost always maintains valid coverage with smaller set sizes.}
\label{tab:marginal_cov_alpha}
\begin{minipage}[t]{\textwidth}
    \resizebox{\textwidth}{!}{%
\begin{tabular}{C{2.2cm}|C{1.7cm}C{1.3cm}C{1.7cm}C{1.3cm}C{1.7cm}C{1.3cm}C{1.7cm}C{1.3cm}C{1.7cm}C{1.3cm}}
\toprule
  $\alpha$ &  \multicolumn{2}{c}{0.05} &  \multicolumn{2}{c}{0.075} &  \multicolumn{2}{c}{0.1} &  \multicolumn{2}{c}{0.15} &  \multicolumn{2}{c}{0.2}\\
\textbf{Pedestrain} & coverage & set size & coverage & set size& coverage & set size& coverage & set size& coverage & set size\\
\ERAPS   &  0.94 &    1.69 &   0.92 &     1.18 & 0.90 &   1.04 &  0.85 &    0.96 & 0.81 &   0.91 \\
\SRAPS   &  0.95 &    4.09 &   0.94 &     3.25 & 0.92 &   3.00 &  0.89 &    2.02 & 0.82 &   1.17 \\
 \SAPS   &  0.95 &    4.29 &   0.93 &     3.77 & 0.91 &   3.00 &  0.86 &    2.16 & 0.81 &   1.86 \\
Naive   &  0.87 &    1.60 &   0.84 &     1.47 & 0.81 &   1.37 &  0.75 &    1.22 & 0.71 &   1.10 \\
\bottomrule
\end{tabular}%
}
\end{minipage}
\begin{minipage}[t]{\textwidth}
    \resizebox{\textwidth}{!}{%
\begin{tabular}{C{2.2cm}|C{1.7cm}C{1.3cm}C{1.7cm}C{1.3cm}C{1.7cm}C{1.3cm}C{1.7cm}C{1.3cm}C{1.7cm}C{1.3cm}}
\toprule
  $\alpha$ &  \multicolumn{2}{c}{0.05} &  \multicolumn{2}{c}{0.075} &  \multicolumn{2}{c}{0.1} &  \multicolumn{2}{c}{0.15} &  \multicolumn{2}{c}{0.2}\\
\textbf{Crop} & coverage & set size & coverage & set size& coverage & set size& coverage & set size& coverage & set size\\
\ERAPS   &  0.96 &    4.68 &   0.93 &     3.53 & 0.90 &   2.87 &  0.86 &    2.22 & 0.82 &   1.80 \\
\SRAPS   &  0.95 &    5.31 &   0.93 &     4.23 & 0.91 &   3.40 &  0.86 &    2.58 & 0.81 &   2.19 \\
 \SAPS   &  0.95 &    4.51 &   0.93 &     3.72 & 0.90 &   3.32 &  0.86 &    2.79 & 0.82 &   2.35 \\
Naive   &  0.96 &    4.65 &   0.94 &     3.97 & 0.92 &   3.50 &  0.88 &    2.88 & 0.83 &   2.46 \\
\bottomrule
\end{tabular}%
}
\end{minipage}
\begin{minipage}[t]{\textwidth}
    \resizebox{\textwidth}{!}{%
\begin{tabular}{C{2.2cm}|C{1.7cm}C{1.3cm}C{1.7cm}C{1.3cm}C{1.7cm}C{1.3cm}C{1.7cm}C{1.3cm}C{1.7cm}C{1.3cm}}
\toprule
  $\alpha$ &  \multicolumn{2}{c}{0.05} &  \multicolumn{2}{c}{0.075} &  \multicolumn{2}{c}{0.1} &  \multicolumn{2}{c}{0.15} &  \multicolumn{2}{c}{0.2}\\
\textbf{Pen digit} & coverage & set size & coverage & set size& coverage & set size& coverage & set size& coverage & set size\\
\ERAPS   &  0.94 &    0.96 &   0.91 &     0.93 & 0.89 &   0.91 &  0.84 &    0.86 & 0.79 &   0.81 \\
\SRAPS   &  0.92 &    0.95 &   0.90 &     0.93 & 0.88 &   0.91 &  0.83 &    0.86 & 0.78 &   0.81 \\
 \SAPS   &  0.94 &    1.03 &   0.92 &     1.00 & 0.90 &   0.96 &  0.84 &    0.89 & 0.79 &   0.83 \\
Naive   &  0.94 &    1.02 &   0.92 &     0.98 & 0.89 &   0.95 &  0.84 &    0.88 & 0.79 &   0.83 \\
\bottomrule
\end{tabular}%
}
\end{minipage}
\end{table}

\begin{figure}[t]
    \centering
   \begin{minipage}[b]{\textwidth}
       \includegraphics[width=\linewidth]{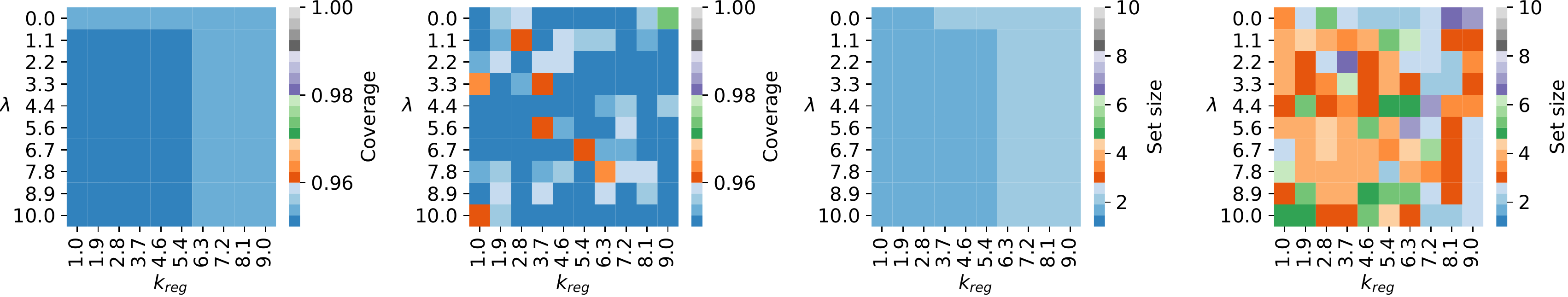}
       \vspace{-0.2in}
       \subcaption{Pedestrain data under NN, 10 classes in total}
       \label{fig:heatmap_Pedestrain}
   \end{minipage}
   \begin{minipage}[b]{\textwidth}
       \includegraphics[width=\linewidth]{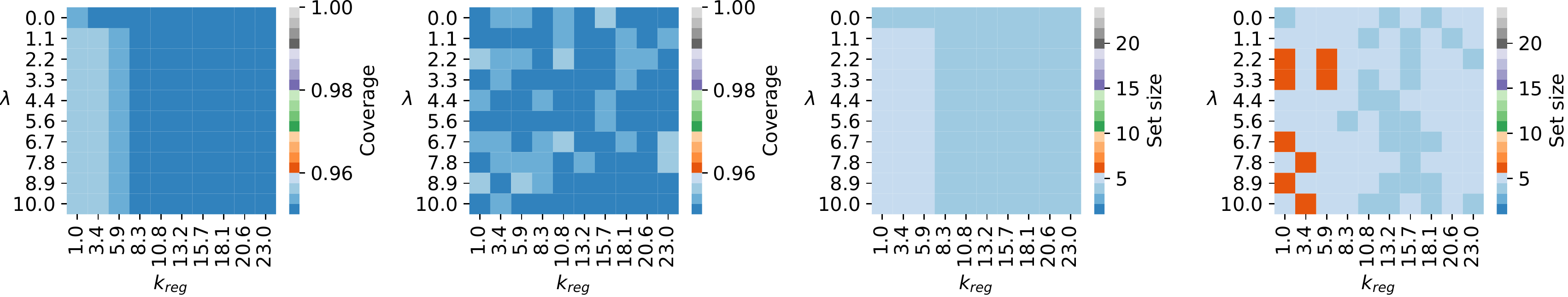}
       \vspace{-0.2in}
       \subcaption{Crop data under NN, 24 classes in total}
       \label{fig:heatmap_Crop}
   \end{minipage}
   \begin{minipage}[b]{\textwidth}
       \includegraphics[width=\linewidth]{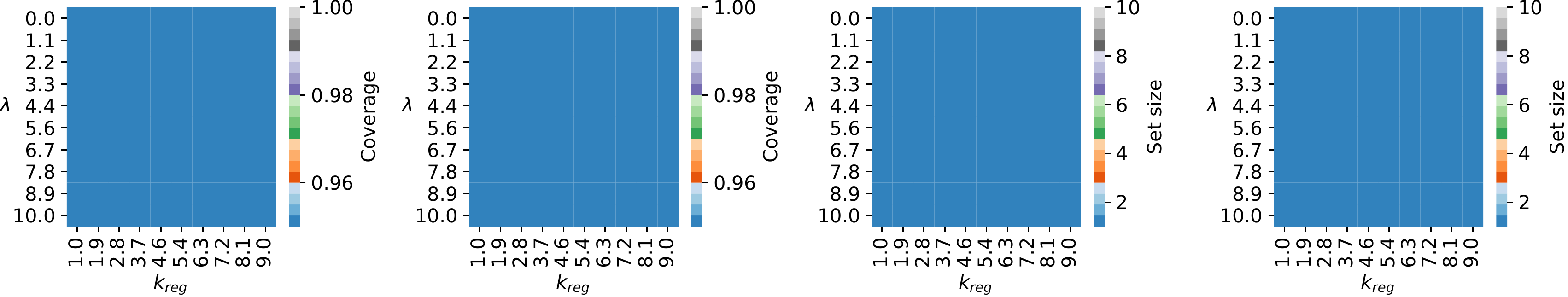}
       \vspace{-0.2in}
       \subcaption{Pen digit data under NN, 10 classes in total}
       \label{fig:heatmap_digit}
   \end{minipage}
   \vspace{-0.2in}
    \cprotect \caption{Marginal coverage (left two) and set sizes (right two) at $\alpha=0.05$ by \ERAPS \ and \SRAPS \ respectively, over 100 pairs of regularizers $(\lambda, k_{reg})$. Both methods maintain valid coverage at any regularizer pairs but \ERAPS \ yields clearly smaller average set sizes in the first two dataset, hence being more informative.}
    \label{fig:marg_cov_regularizers}
\end{figure}

We first show the performance of different methods over various $\alpha \in [0,1]$ in Table \ref{tab:marginal_cov_alpha}, which examines coverage of both methods over different target coverage at fixed a fixed pair of regularizer $(\lambda,k_{reg})=(1,2)$. The marginal coverage (resp. set size) is computed over all $T_1$ test data by checking whether each test datum is included within the prediction set (resp. computing set size). It is clear that \ERAPS \ maintains valid coverage at all times, with smaller set sizes than other methods on all but the PenDigits dataset. 

On the other hand, recall that the regularizers $\{\lambda,k_{reg}\}$ are important in controlling set sizes, so Figure \ref{fig:marg_cov_regularizers} examines the marginal results over 100 different pairs of $(\lambda,k_{reg})$ at a fixed $\alpha=0.05$. The pairs come from the Cartesian product of 10 uniformly spaced $\lambda$ (resp. $k_{reg}$) over $[0.01, 10]$ (resp. $[1, K-1]$). On the first two datasets in Figure \ref{fig:heatmap_Pedestrain} and \ref{fig:heatmap_Crop}, \SRAPS \ sometimes yields much more conservative coverage than \ERAPS \ by producing prediction sets that almost contain all labels. Such sets can be too big to be useful in practice. In contrast, \ERAPS \ typically produces sets of size between 2 to 5 that are more precise. The results are also very stable over different regularization parameters, thus relying less on parameter tuning. On the last data in Figure \ref{fig:heatmap_digit}, both methods have nearly identical results. We suspect this happens because the dataset resembles typical image classification datasets, in which data are less dependent so that existing CP methods already perform well \citep{MJ_classification,Candes_classification}. Overall, the effect of different regularizers seems minimal on the empirical coverage by \ERAPS, while for a fixed $\lambda$, larger $k_{reg}$ corresponds to weaker regularization on set sizes so that sets tend to increase in size as $k_{reg}$ increases. Hence, we suggest picking $k_{reg}$ relatively small (e.g., 1 or 2) so that the sets tend to be smaller; $\lambda$ can be set as 1. Note that \citep[Appendix E]{MJ_classification} provides guidance on computing these parameters to optimize set sizes, but doing so requires a separate set of \textit{tuning data}, which may not be feasible when data are scarce. 
\begin{table}[t]
\centering
\cprotect\caption{MelbournePedestrian data with 10 classes: Class-conditional coverage (top two tables) and set size (bottom two tables) of \ERAPS \ and \SRAPS, where we condition on different road types. Both methods maintain valid coverage, but \ERAPS \ produces much smaller sets than \SRAPS \ at each class.}
\label{tab:cond_cov_MelbournePedestrian}
\resizebox{\textwidth}{!}{%
\begin{tabular}{p{1.5cm}p{1.5cm}p{1.5cm}p{1.5cm}p{1.5cm}p{1.5cm}p{1.5cm}p{1.5cm}p{1.5cm}p{1.5cm}p{1.5cm}}
\toprule
\multicolumn{11}{c}{\ERAPS: Conditional coverage for different pedestrian road types} \\
  $\alpha$ & class 0 &  class 1 &  class 2 &  class 3 &  class 4 &  class 5 &  class 6 &  class 7 &  class 8 &  class 9 \\
 0.05 &     0.96 &     0.98 &     0.94 &     0.99 &     0.99 &     0.91 &     1.00 &     0.83 &     0.89 &     0.95 \\
  0.075 &     0.95 &     0.98 &     0.89 &     0.99 &     0.98 &     0.88 &     1.00 &     0.76 &     0.86 &     0.93 \\
  0.1 &     0.95 &     0.97 &     0.78 &     0.98 &     0.97 &     0.86 &     0.97 &     0.72 &     0.82 &     0.93 \\
  0.15 &     0.94 &     0.92 &     0.74 &     0.92 &     0.89 &     0.83 &     0.89 &     0.69 &     0.77 &     0.89 \\
  0.2 &     0.88 &     0.86 &     0.73 &     0.84 &     0.84 &     0.77 &     0.84 &     0.64 &     0.76 &     0.86 \\
\hline
\multicolumn{11}{c}{\SRAPS: Conditional coverage for different pedestrian road types} \\
  $\alpha$ & class 0 &  class 1 &  class 2 &  class 3 &  class 4 &  class 5 &  class 6 &  class 7 &  class 8 &  class 9 \\
0.05 & 0.98 & 0.97 & 0.99 & 1    & 0.98 & 0.82 & 1    & 0.76 & 0.99 & 0.98 \\
0.08 & 0.98 & 0.94 & 0.98 & 1    & 0.97 & 0.8  & 1    & 0.68 & 0.97 & 0.98 \\
0.1  & 0.98 & 0.93 & 0.95 & 1    & 0.95 & 0.75 & 1    & 0.66 & 0.9  & 0.95 \\
0.15 & 0.95 & 0.91 & 0.84 & 1    & 0.94 & 0.67 & 0.98 & 0.65 & 0.7  & 0.86 \\
0.2  & 0.94 & 0.9  & 0.73 & 1    & 0.92 & 0.59 & 0.98 & 0.64 & 0.61 & 0.8  \\
\hline
\end{tabular}%
}
\resizebox{\textwidth}{!}{%
\begin{tabular}{p{1.5cm}p{1.5cm}p{1.5cm}p{1.5cm}p{1.5cm}p{1.5cm}p{1.5cm}p{1.5cm}p{1.5cm}p{1.5cm}p{1.5cm}}
\hline
\multicolumn{11}{c}{\ERAPS: Conditional set size for different pedestrian road types} \\
  $\alpha$ &  class 0 &  class 1 &  class 2 &  class 3 &  class 4 &  class 5 &  class 6 &  class 7 &  class 8 &  class 9 \\
0.05 &     2.03 &     2.14 &     2.17 &     1.93 &     1.07 &     1.22 &     1.00 &     1.35 &     2.04 &     2.01 \\
  0.075 &     1.15 &     1.21 &     1.66 &     1.07 &     1.01 &     1.07 &     1.00 &     1.14 &     1.52 &     1.11 \\
  0.1 &     1.05 &     1.05 &     1.16 &     1.02 &     0.99 &     1.00 &     0.97 &     1.00 &     1.19 &     1.05 \\
  0.15 &     1.00 &     0.96 &     1.00 &     0.95 &     0.91 &     0.98 &     0.89 &     0.98 &     1.00 &     0.97 \\
  0.2 &     0.93 &     0.90 &     0.99 &     0.87 &     0.86 &     0.91 &     0.84 &     0.92 &     0.99 &     0.93 \\
  \hline
\multicolumn{11}{c}{\SRAPS: Conditional set size for different pedestrian road types} \\
  $\alpha$ &  class 0 &  class 1 &  class 2 &  class 3 &  class 4 &  class 5 &  class 6 &  class 7 &  class 8 &  class 9 \\
0.05 & 4    & 4.28 & 4.52 & 4    & 4.08 & 4.1  & 4    & 4.4  & 4.96 & 4    \\
0.08 & 3.05 & 3.3  & 3.94 & 3    & 3.41 & 3.66 & 3.07 & 3.86 & 4    & 3.06 \\
0.1  & 2.06 & 2.3  & 2.94 & 2    & 2.36 & 2.69 & 2.12 & 2.78 & 3    & 2.12 \\
0.15 & 1.06 & 1.29 & 1.8  & 1.02 & 1.26 & 1.55 & 1.06 & 1.54 & 2    & 1.16 \\
0.2  & 1    & 1.24 & 1.27 & 1.01 & 1.05 & 1.12 & 1.01 & 1.34 & 1.83 & 1   \\
\bottomrule
\end{tabular}%
}

\end{table}

\subsection{Conditional results} \label{exp:conditional}
\begin{table}[t]
\centering
\cprotect\caption{Set-stratified coverage of \ERAPS \ and \SRAPS. The second column onward is indexed by the sizes of set strata. NaN entries indicate no sets with sizes within this strata; we prefer more NaN entries at larger set strata. Both \ERAPS \ and \SRAPS \ could maintain valid set-stratified coverage, while the former yields smaller prediction sets.}
\label{tab:set_stratified_cov}
\resizebox{\textwidth}{!}{%
\begin{tabular}{p{1.5cm}p{1.5cm}p{1.5cm}p{1.5cm}p{1.5cm}p{1.5cm}|p{1.5cm}p{1.5cm}p{1.5cm}p{1.5cm}p{1.5cm}}
\toprule
& \multicolumn{5}{c|}{\ERAPS \ on Pedestrian} & \multicolumn{5}{c}{\SRAPS \ on Pedestrian} \\
  $\alpha$ & $[0,2)$ &  $[2,4)$ &  $[4,6)$ &  $[6,8)$ &  $[8,10)$ & $[0,2)$ &  $[2,4)$ &  $[4,6)$ &  $[6,8)$ &  $[8,10)$ \\
0.05 &  0.97 &  0.93 & NaN & NaN & NaN &   NaN &   NaN &  0.95 & NaN & NaN \\
0.075 &  0.95 &  0.79 & NaN & NaN & NaN &   NaN &  0.98 &  0.88 & NaN & NaN \\
0.1 &  0.91 &  0.70 & NaN & NaN & NaN &   NaN &  0.91 &   NaN & NaN & NaN \\
0.15 &  0.85 &   NaN & NaN & NaN & NaN &  0.94 &  0.70 &   NaN & NaN & NaN \\
0.2 &  0.81 &   NaN & NaN & NaN & NaN &  0.87 &  0.57 &   NaN & NaN & NaN \\
\hline 
& \multicolumn{5}{c|}{\ERAPS \ on Crop} & \multicolumn{5}{c}{\SRAPS \ on Crop} \\
  $\alpha$ & $[0,4)$ &  $[4,9)$ &  $[9,14)$ &  $[14,19)$ &  $[19,24)$ & $[0,4)$ &  $[4,9)$ &  $[9,14)$ &  $[14,19)$ &  $[19,24)$ \\
0.05 &  1.00 &  0.95 & NaN & NaN & NaN &   NaN &  0.95 & NaN & NaN & NaN \\
0.075 &  0.98 &  0.90 & NaN & NaN & NaN &   NaN &  0.93 & NaN & NaN & NaN \\
0.1 &  0.92 &  0.82 & NaN & NaN & NaN &  0.98 &  0.84 & NaN & NaN & NaN \\
0.15 &  0.86 &  0.39 & NaN & NaN & NaN &  0.86 &   NaN & NaN & NaN & NaN \\
0.2 &  0.82 &   NaN & NaN & NaN & NaN &  0.81 &   NaN & NaN & NaN & NaN \\
\hline 
& \multicolumn{5}{c|}{\ERAPS \ on Pen digits} & \multicolumn{5}{c}{\SRAPS \ on Pen digits} \\
  $\alpha$ & $[0,2)$ &  $[2,4)$ &  $[4,6)$ &  $[6,8)$ &  $[8,10)$ & $[0,2)$ &  $[2,4)$ &  $[4,6)$ &  $[6,8)$ &  $[8,10)$ \\
0.05 &  0.94 & NaN & NaN & NaN & NaN &  0.92 & NaN & NaN & NaN & NaN \\
0.075 &  0.91 & NaN & NaN & NaN & NaN &  0.89 & NaN & NaN & NaN & NaN \\
0.1 &  0.89 & NaN & NaN & NaN & NaN &  0.87 & NaN & NaN & NaN & NaN \\
0.15 &  0.84 & NaN & NaN & NaN & NaN &  0.82 & NaN & NaN & NaN & NaN \\
0.2 &  0.79 & NaN & NaN & NaN & NaN &  0.77 & NaN & NaN & NaN & NaN \\
\bottomrule
\end{tabular}
}

\end{table}
We consider two types of conditional coverage for different goals. The first is the \textit{class-conditional} coverage as defined in Eq. \eqref{def:class_cond_cov}, where we aim to reach uniform coverage over all classes, regardless of its frequency $P(Y)$. Valid class-conditional coverage means the coverage is at least $1-\alpha$ over all classes. The second is similar to the \textit{set-stratified} conditional coverage \citep[Section 4]{MJ_classification}, where we assess how sets of different sizes undercover or overcover. The set-stratified coverage results better inform decision-making given the fixed sizes of the prediction sets. The ideal case for set-stratified coverage is that all sets are small (e.g., sizes are close to 1) and maintain valid coverage at each size strata. In particular, we do not use the class-conditional non-conformity scores as in Remark \ref{remark:class_cond} because the scores heavily depend on the quality of estimation and/or the similarity in test and training distribution of non-conformity scores. When either of these requirements breaks down, we observe under-coverage at certain labels because the class-conditional thresholds based on training data become too small. The same issue did not occur when we used $\hat \tau_{cal}$, the marginal threshold, likely because classes with large non-conformity scores raise the overall threshold, thus increasing set sizes to maintain valid class-conditional coverage. We did not use $\hat{\tau}^{\max}_{cal}$ because the sets tend to be too conservative.

Table \ref{tab:cond_cov_MelbournePedestrian} shows the class-conditional results for the Pedestrian dataset\footnote{The classes correspond to Bourke Street Mall, Southern Cross Station, New Quay, Flinders St Station Underpass, QV Market-Elizabeth, Convention/Exhibition Centre,  Chinatown-Swanston St, Webb Bridge, Tin Alley-Swanston St, Southbank.}, where we compare between \ERAPS \ and \SRAPS \ and note that the comparisons against \SAPS \ and Naive are similar. \ERAPS \ tends to maintain valid class-conditional coverage at most classes with similar set sizes, whereas \SRAPS \ can be too conservative. Table \ref{tab:set_stratified_cov} shows the set-stratified conditional results for all three datasets. We can see that nearly all sets by \ERAPS \ are small, and it reaches valid coverage on most of its strata; there is significant under-coverage in certain cases (e.g., $\alpha=0.15$ with strata = $[4,9)$ on crop data) due to the possible existence of outliers. However, as we have seen from earlier results, the overall coverage is likely unaffected as outliers are rare. Similar performance is observed when using \SRAPS, which may sometimes yield large prediction sets and be conservative in terms of coverage.





\section{Conclusion}\label{sec:conclude}
In this work, we extend techniques in \citep{EnbPI} for building distribution-free, ensemble-based prediction sets for classification problems. Theoretically, we bound coverage gaps of the estimated prediction sets and demonstrate set convergence under further assumptions. Empirically, \ERAPS \ tends to yield smaller prediction sets and valid marginal and conditional coverage. In the future, we aim to more systematically compare \ERAPS \ against other competing conformal prediction baselines on more benchmark datasets. We also want to extend beyond the worst-case analyses of coverage gaps with relaxed assumptions. Lastly, we believe it is important to explre other designs of non-conformity scores beyond \citep{MJ_classification} for better practical performance under various situations. 

\section{Proof}\label{sec:proof}

\begin{proof}[Proof of Lemma \ref{lem:tildeFandhatF}]
The proof is identical to that of \citep[Lemma 2]{EnbPI} so we omit the mathematical details. The gist of the proof proceeds by bounding the size of the set of past $T$ estimated non-conformity scores which deviate too much from the oracle one. The set is denoted as
\[
S_T:=\{i \in [T]: |\hat \tau_i-\tau_i|>\gamma_T^{2/3} \}.
\]
Then, one can relate the difference $|\tildeF{x}-\hatF{x}|$ at each $x$ to a sum of two terms of indicator variables--ones whose index belongs to $S$ and ones which does not. The ones that does not belong to $S$ can be bounded using the term $|\tildeF{x}-F(x)|$ up to a multiplicative constant.
\end{proof}

\begin{proof}[Proof of Lemma \ref{lem:tildeFandF}]
The proof is identical to that of \citep[Lemma 1]{EnbPI} so we omit the mathematical details. In fact, this is a simple corollary of the famous Dvoretzky–Kiefer–Wolfowitz
inequality \citep[p.210]{Kosorok2008IntroductionTE}, which states the convergence of the empirical bridge to actual distributions under the i.i.d. assumption.
\end{proof}

\begin{proof}[Proof of Theorem \ref{thm:asym_cond_cov}]
The proof is identical to that of \citep[Theorem 1]{EnbPI} so we omit the mathematical details. The gist of the proof proceeds by bounding the non-coverage $|\PP(Y_t \notin C(X_t,\alpha)|X_t=x_t)-\alpha|$ using the sum of constant multiples of $\sup_{x}|\tildeF{x}-\hatF{x}|$ and $\sup_{x}|\tildeF{x}-F(x)|$, both of which can be bounded by Lemmas \ref{lem:tildeFandhatF} and \ref{lem:tildeFandF} above.
\end{proof}

\begin{proof}[Proof of Theorem \ref{thm:asy_set}]
Based on the assumptions and the definition in \eqref{eq:set}, we now have
\begin{align*}
    C^*(X_t,\alpha)&=\{1,\ldots,c^*\}, c^*=\arg\max_{c} \tau_t(c) < F^{-1}(1-\alpha),\\
    C(X_t,\alpha)&=\{1,\ldots,\hat{c}\}, \hat{c}=\arg\max_{c} \hat{\tau}_t(c) < \hat{F}^{-1}(1-\alpha),
\end{align*}
where $\hat{F}^{-1}$ is the empirical CDF based on estimated non-conformity scores $\{\hat{\tau}_{t-T},\ldots,\hat{\tau}_{t-1}\}$.
We now show that $ C(X_t,\alpha) \Delta C^*(X_t,\alpha) \leq 1$ if and only if 
\[
\| \hat{\tau}_t- \tau_t\|_{\infty} \rightarrow 0 \text{ and }\hat{F}^{-1}(1-\alpha)\rightarrow F^{-1}(1-\alpha).
\]
($\Rightarrow$) Without loss of generality, suppose that $\hat{c}<c^*$ so that $ C(X_t,\alpha) \Delta C^*(X_t,\alpha) > 1$. Then, by definition of the prediction sets, we must have 
\begin{align*}
    & \hat{\tau}_t(c^*) \geq \hat{F}^{-1}(1-\alpha), \\
    & \tau_t(c^*) < F^{-1}(1-\alpha).
\end{align*}
Denote $\delta_{\tau,t}:=\hat{\tau}_t(c^*)-\tau_t(c^*)$ and $\delta_{F,t}:=F^{-1}(1-\alpha)-\hat{F}^{-1}(1-\alpha)$, we thus have 
\[
\delta_{\tau,t}+\delta_{F,t}\geq F^{-1}(1-\alpha)-\tau_t(c^*)>0.
\]
However, this is a contraction when $T$ approaches infinity---by the assumption that $\| \hat{\tau}_t- \tau_t\|_{\infty} \rightarrow 0$ and the earlier results that $\hat{F}^{-1}(1-\alpha)\rightarrow F^{-1}(1-\alpha)$, we must have $\delta_{\tau,t}$ and $\delta_{F,t}$ both converging to zero.

\noindent ($\Leftarrow$) By the form of the estimated and true prediction sets, it is obvious that if $\| \hat{\tau}_t- \tau_t\|_{\infty} \rightarrow 0$ and $\hat{F}^{-1}(1-\alpha)\rightarrow F^{-1}(1-\alpha)$, their set difference must converges to zero.
\end{proof}

\bibliography{References}
\bibliographystyle{plainnat}

\end{document}